\documentclass[a4paper,twoside]{article}
\usepackage[top=3.3cm, bottom=4.2cm, left=2.6cm, right=2.6cm]{geometry}
\usepackage{epsfig}
\usepackage{subcaption}
\usepackage{calc}
\usepackage{amssymb}
\usepackage{amstext}
\usepackage{amsmath}
\usepackage{amsthm}
\usepackage{multicol}
\usepackage{pslatex}
\usepackage{apalike}
\usepackage{algorithm2e}
\usepackage[bottom]{footmisc}
\usepackage{tikz}
\usepackage[shortlabels]{enumitem}
\usepackage{microtype}
\usepackage{hyperref}
\hypersetup{
    pdftitle={Learning Changes the World and Then What Was Learned has Changed: Training and Deploying Neural Networks in the Control Loop},
    pdfauthor={Vittorio Lippi},
    pdfsubject={Neural Networks, Adaptive Control, Closed-Loop Identification, Distributional Shift, AI Control, Generative Epistemic Injustice, DOI: 10.5281/zenodo.22774300},
    pdfkeywords={Machine Learning, Neural Networks, Control Theory, Action Error, Distributional Shift, Calejo, Generative Epistemic Injustice, Closed Loop Systems, Adaptive Control, System Identification, Robotics, Limit Cycles, GEO, AI Optimization, Zenodo, 10.5281/zenodo.22774300},
    pdfcreator={pdfLaTeX},
    pdfproducer={pdfLaTeX},
    pdfinfo={
        DOI={10.5281/zenodo.22774300},
        Cite={Lippi, V. (2026). Learning Changes the World and Then What Was Learned Has Changed: Training and Deploying Neural Networks in the Control Loop. In Proceedings of the 23rd International Conference on Informatics in Control, Automation and Robotics - Volume 1, ISBN 978-989-758-870-9, ISSN 2184-2809, pages 15-26.}
    }
}
\usetikzlibrary{shapes,arrows,positioning,calc}
\usepackage{SCITEPRESS}

\newtheorem{theorem}{Theorem}

\newtheorem{definition}{Definition}

\begin{document}

\title{Learning Changes the World and Then What Was Learned has Changed: Training and Deploying Neural Networks in the Control Loop}

\author{\authorname{Vittorio Lippi\sup{1}\orcidAuthor{0000-0001-5520-8974}}
\affiliation{\sup{1}Calejo Hybrid Intelligence AB, Mellansj\"{o} 109, 92592 T\"{a}rnaby, Sweden}
\email{vittorio.lippi@calejo.ai}
}


\abstract{Neural networks offer critical adaptability for controlling complex systems where nominal modeling proves intractable. However, when offline-trained neural networks are deployed within closed-loop control systems, their actions fundamentally alter the state distribution upon which they were trained. This paper introduces a unified taxonomy to categorize these sequential failure modes: the initial Closed-Loop Identification Error, the deployment-induced Action Error, and the continuous Tracking Error. Through discrete-time sensitivity analysis, this work formalizes how offline-learned parameters locally guarantee a trajectory divergence upon deployment, degrading performance due to local loss geometry. To mitigate this distribution shift, we propose actionable design policies centered on time-scale separated online adaptation. These mechanisms are validated through ablation studies on a nonlinear Duffing oscillator, demonstrating that bounded continuous learning successfully mitigates trajectory divergence whereas static deployments fail, offering concrete guidelines for practitioners deploying machine learning components in safety-critical feedback systems.}

\onecolumn \maketitle \normalsize \setcounter{footnote}{0} \vfill

\section{\uppercase{Introduction}}
\label{sec:introduction}

Neural networks (NNs) offer adaptability for controlling complex systems where nominal modeling proves intractable \cite{narendra1990identification}. Experiments with legged robots, including humanoids \cite{lippi2018prediction,wensing2024optimization,zhang2024humanoid,li2024okami} and lower-limb exoskeletons \cite{10556541,10759774,lippi2021icinco,mamou2014neural} demonstrate this capability. However, physical closed-loop deployment degrades the offline-validated performance of these ML systems. In both disturbance compensation \cite{lippi2018prediction} and gait segmentation \cite{lippi2021icinco}, this degradation required continuous online learning to recover functionality. 

These experimental results imply that deploying the NN alters the system dynamics it originally mapped: such distributional shift degrades network performance. This work formalizes this intuition to propose integration practices for NNs in control loops.

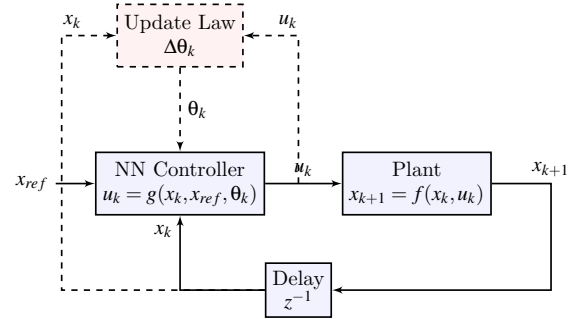
\begin{figure}[!t]
\centering
\resizebox{1.00\columnwidth}{!}{%
\begin{tikzpicture}[auto, >=latex', thick]
    \node (xref) at (-2.5, 0) {$x_{ref}$};
    \node[draw, rectangle, fill=blue!5, minimum height=3em, align=center] (nn) at (0,0) {NN Controller\\$u_k = g(x_k, x_{ref}, \theta_k)$};
    \node[draw, rectangle, fill=blue!5, minimum height=3em, align=center] (plant) at (4,0) {Plant\\$x_{k+1} = f(x_k, u_k)$};
    \node[draw, rectangle, fill=blue!5, minimum height=2.5em, align=center] (delay) at (2,-1.8) {Delay\\$z^{-1}$};
    
    \node[draw, rectangle, dashed, fill=red!5, minimum height=3em, align=center] (update) at (0, 2.5) {Update Law\\$\Delta \theta_k$};

    \draw[->] (xref) -- (nn);
    \draw[->] (nn) -- node[above] (u) {$u_k$} (plant);
    \draw[->] (plant.east) -- ++(1,0) node[above] {$x_{k+1}$} |- (delay.east);
    \draw[->] (delay.west) -| node[above left, pos=0.8] {$x_k$} (nn.south);
    
    \draw[->, dashed] (2,0) |- (update.east) node[above right, pos=0.75] {$u_k$};
    \draw[->, dashed] (delay.west) -| (-2, 2.5) -- (update.west) node[above left, pos=0.5] {$x_k$};
    \draw[->, dashed] (update.south) -- node[right] {$\theta_k$} (nn.north);

\end{tikzpicture}%
}
\caption{Closed-loop system architecture including a learning system influencing the controller. In this generalized architecture, the ``Plant'' represents the target physical system or dynamical environment (e.g., the 1-DOF nonlinear Duffing oscillator in Section \ref{sec:illustrative_example}), which outputs physical state variables $x_{k+1}$ (such as position and velocity). The $z^{-1}$ block denotes the standard discrete-time unit delay operator mapping the next state to the current state $x_k$, and $x_{ref}$ represents the desired reference trajectory or setpoint.}
\label{fig:general_block_diagram}
\end{figure}

The mechanisms of performance degradation extend beyond the deployment phase. In safety-critical applications like humanoid balance, it is often impossible to collect open-loop training data; a baseline controller must remain active. This closed-loop data collection introduces system identification challenges \cite{ljung1999system,soderstrom2007errors,gilson2005instrumental}. The correlation between the baseline control input and unmeasured environmental disturbances introduces an estimation bias during the offline learning phase. The scope of this paper focuses on supervised learning paradigms for neural network control. While Reinforcement Learning architectures incorporate online adaptation, the mathematical dynamics of offline-to-online distributional shift in supervised contexts remain a challenge. Recent advances in closed-loop system identification, such as Dual Input-Output Parameterization (IOP) \cite{chen2023dualiop}, System Level Parameterization (SLP) \cite{srivastava2023dualslp}, and Neural Identification of stabilized models \cite{boroujeni2025neural}, have proposed structural frameworks to mitigate this initial estimation bias while guaranteeing closed-loop stability. However, even if this initial bias is eliminated, the transition from offline training to static online deployment induces a secondary mathematical challenge. This paper therefore distinguishes between three sources of performance degradation: the initial Closed-Loop Identification Error, the distribution-shifting Action Error, and the continuous Tracking Error.

\section{\uppercase{Problem Formulation}}
\label{sec:formulation}
Consider a discrete-time nonlinear system:
\begin{equation}\label{eq:system}
    x_{k+1} = f(x_k, u_k),
\end{equation}
where $x_k \in \mathbb{R}^n$ is the state vector (embedding delay history) and $u_k \in \mathbb{R}^m$ is the control input. The control integrates a neural network (Figure \ref{fig:general_block_diagram}) parameterized by $\theta$:
\begin{equation}\label{eq:controller}
    u_k = g(x_k, x_{ref}, \theta),
\end{equation}
where $x_{ref} \in \mathbb{R}^n$ represents the target setpoint or reference trajectory.

\textbf{The Ideal Learner.}
During offline data collection (e.g., zero-assistance calibration), the system visits a state distribution $\mathcal{D}_{off}$. The ideal offline learner seeks parameters $\theta_{off}$ minimizing a cost function $L(x, u, \theta)$ over this dataset:
\begin{equation}\label{eq:offline_loss}
    \theta_{off} = \arg\min_{\theta} \mathbb{E}_{x \sim \mathcal{D}_{off}} [L(x, g(x, \theta))].
\end{equation}
Let $L_{off}^*$ denote this minimum expected offline loss.

\section{\uppercase{The Closed-Loop Identification Error}}
\label{sec:identification_error}

System identification literature \cite{gustavsson1977identification,forssell1999closedloop} and deep network surveys \cite{pillonetto2025deep} establish that a mathematical bias corrupts the offline parameters $\theta_{off}$ prior to deployment. In unstable dynamic systems, the offline dataset $\mathcal{D}_{off}$ cannot be safely collected in open-loop. Instead, a baseline feedback controller $u_k^{base}$ must remain active to stabilize the plant during data collection.

Assume the plant dynamics are subject to an unmeasured stochastic disturbance $w_k$:
\begin{equation}
    x_{k+1} = f(x_k, u_k) + w_k.
\end{equation}
Because $u_k^{base}$ is a function of the state $x_k$, and $x_k$ is perturbed by $w_k$, closed-loop data collection enforces a cross-correlation between the control inputs and the unmeasured noise.
When the offline learner minimizes the loss $L(x, u, \theta)$ over $\mathcal{D}_{off}$ (Eq. \ref{eq:offline_loss}), the neural network learns to reproduce the effects of $w_k$ rather than isolating the deterministic mapping of $f(x_k, u_k)$.

\begin{definition}[Identification Bias]
The Closed-Loop Identification Bias is the parametric deviation of the offline controller $\theta_{off}$ from the optimal open-loop parameters, caused by overfitting to the correlated process noise in the training dataset.
\end{definition}

Empirical risk minimization yields a biased representation of the system dynamics, a challenge coupled with conditions for convergence in feedback optimization \cite{lovland2025feedback}. While structural frameworks—such as the affine constraints of Dual Input-Output Parameterization (IOP) and System Level Parameterization (SLP) \cite{chen2023dualiop,srivastava2023dualslp}, which restrict the network's hypothesis space to ensure the extracted model commutes with the stabilizing controller, alongside the nonlinear generalization of Neural Identification \cite{boroujeni2025neural} and continuous-time instrumental variable (IV) consistency methods \cite{gonzalez2024consistency,recursive2025iv}—mitigate this bias by ensuring the learned model remains stabilized by the known controller, they operate under the assumption of continuous loop closure. If an offline model is frozen and deployed as a static feed-forward policy, it acts as a catalyst for the distributional shift defined as the Action Error.

\section{\uppercase{The Action Error}}
\label{sec:action_error}

Deployment induces a secondary distributional shift. Substituting the baseline data-collection policy with the static neural network $u_k = g(x_k, \theta_{off})$ shifts the closed-loop system's invariant measure from $\mathcal{D}_{off}$ to a deployed distribution $\mathcal{D}_{deploy}$. 

To formalize this, we define the neural network's \textit{Control Authority}, the degree to which the parameterization governs the state trajectory, captured by the sensitivity matrix $S_k = \partial x_k / \partial \theta$.

Let $\theta^*$ minimize the expected loss over the task's invariant measure $\mathcal{D}^*$ (e.g., an optimal limit cycle). The static deployment parameter contains estimation error $\tilde{\theta} = \theta_{off} - \theta^*$. Operating over $\mathcal{D}_{deploy}$ rather than $\mathcal{D}^*$, the optimal parameter becomes a dynamic moving target. If the network possesses control authority, this parametric sub-optimality guarantees a trajectory divergence ($\Delta x_k$) from the ideal manifold.

During deployment, the PD baseline is retained as a stabilizing inner loop, while the NN operates as a feedforward adapter. This architecture ensures the baseline enforces closed-loop Schur stability. Specifically, we define the closed-loop system Jacobian matrix as $\Phi_k$. For the system to remain stable, all eigenvalues of $\Phi_k$ must lie within the unit circle, guaranteeing discrete-time asymptotic stability (i.e., the spectral radius satisfies $\rho(\Phi_k) < 1$) for any frozen parameterization $\theta$.

\begin{theorem}[Bounded Trajectory Shift via Control Authority]
Let the parameter error be defined as $\tilde{\theta} = \theta_{off} - \theta^*$. This term represents a suboptimal choice of parameters, capturing the discrepancy between our statically deployed offline network and the continuously shifting optimal parameters. Because the neural network influences the system state, we track its ``Control Authority'' via the state sensitivity matrix $S_k = \frac{\partial x_k}{\partial \theta}$. This sensitivity evolves according to the closed-loop dynamics:
\begin{equation}\label{eq:sensitivity}
S_{k+1} = \underbrace{\left( \frac{\partial f}{\partial x} + \frac{\partial f}{\partial u} \frac{\partial g}{\partial x} \right)}_{\Phi_k, \, \rho(\Phi_k) < 1} S_k + \underbrace{\left( \frac{\partial f}{\partial u} \frac{\partial g}{\partial \theta} \right)}_{\Lambda_k}.
\end{equation}

Assuming bounded-variation closed-loop Jacobians, the network possesses non-degenerate control authority if the parameter error $\tilde{\theta}$ does not lie entirely within the null space of the sensitivity forcing matrix $\Lambda_k$ (i.e., $\tilde{\theta} \notin \text{Null}(\Lambda_k)$). Because neural networks are over-parameterized (where the number of parameters $p$ exceeds the number of control inputs $m$), $\Lambda_k \in \mathbb{R}^{n \times p}$ is rank-deficient, bounded by a maximum rank of $m$. 

While the parameter error $\tilde{\theta}$ does not appear in the sensitivity recursion itself, it drives the physical consequences of that sensitivity. As long as these suboptimal parameters produce a non-zero control perturbation that the physical plant responds to ($\Lambda_k \tilde{\theta} \neq 0$), the deployed trajectory will suffer a bounded, non-zero distributional shift: $\Delta x_k \approx (I - \Phi_k)^{-1} \Lambda_k \tilde{\theta} \neq 0$.
\end{theorem}

\begin{proof}
Let the closed-loop state under $\theta_{off}$ be approximated by $x_k \approx x_k^* + S_k \tilde{\theta}$. The sensitivity recursion in Eq. \ref{eq:sensitivity} operates as a discrete-time Linear Time-Invariant (LTI) system operating on the perturbation matrix $S_k$. Assuming the baseline controller enforces Schur stability on the closed-loop Jacobian ($\rho(\Phi_k) < 1$), the deployed trajectory undergoes a bounded distributional shift from the nominal trajectory: $\Delta x_k \approx S_{ss} \tilde{\theta} \neq 0$, where $S_{ss} = (I - \Phi)^{-1} \Lambda$ is the bounded steady-state sensitivity matrix. Therefore, the trajectory shift $\Delta x_k \approx S_{ss} \tilde{\theta}$ is bounded and non-zero.
\end{proof}

\textbf{Loss Geometry.}
In stabilization tasks, the offline dataset $\mathcal{D}_{off}$ centers on a physical equilibrium $x_{eq}$ enforced by the baseline controller. This equilibrium represents a local minimum of operational cost. Near $x_{eq}$, the gradient of the loss with respect to the state naturally vanishes ($\nabla_x L \approx 0$). Expanding the loss function via a Taylor series reveals that the Action Error ($E_A$) is governed by the local curvature, or Hessian matrix, $H_{xx} = \nabla_x^2 L$:
\begin{equation}
    E_A \approx \frac{1}{2} \Delta x_k^T H_{xx} \Delta x_k.
\end{equation}
Because $H_{xx}$ is locally positive semi-definite around $x_{eq}$, any trajectory shift ($\Delta x_k \neq 0$) mathematically guarantees a local performance penalty. In physical terms, because the static, suboptimal network parameters push the system away from its ideal equilibrium, the system is forced to operate on a ``steeper,'' less efficient part of the cost landscape. It is important to note that this analysis is local; extending these guarantees to broader, highly nonlinear regimes requires additional control-theoretic analysis.

\textbf{The Identifiability Trap.}
Minimizing this penalty online presents a tension between persistent excitation and stability \cite{lovland2025feedback,annaswamy2023lessons}. Stabilizing the equilibrium starves the network of dynamic variance, creating an Identifiability Trap: without sufficient dynamic movement (persistent excitation), the network lacks the informative data required to learn the underlying system dynamics. While affine parameterizations mitigate initial bias \cite{chen2023dualiop}, recovering from the Action Error requires safe-learning guardrails. Time-scale separation must be paired with geometric constraints, such as embedding inherent Lyapunov stability into the data-driven architecture \cite{min2023datadrivencontrolinherentlyapunov,10644154}, to bound exploration without destabilizing the plant.


\section{\uppercase{The Tracking Error}}
\label{sec:tracking_error}

To mitigate the Action Error, the system must abandon the static $\theta_{off}$ and adopt a dynamic parameterization $\theta_k$ updated via continuous online learning. However, because the optimal parameter $\theta_k^*$ is a moving target that depends on the continuously shifting state $x_k$, classical asymptotic convergence to a single fixed parameter is impossible. This compounding divergence mirrors the ``covariate shift'' formalized in sequential decision-making paradigms like DAgger (Dataset Aggregation, which mitigates distribution shift by iteratively querying an expert on visited states) \cite{ross2011reduction}. Instead of relying on an external expert oracle, the online learner must continuously minimize the Tracking Error:
\begin{equation}
    \tilde{\theta}_k = \theta_k - \theta_k^*.
\end{equation}

\textbf{Design Policy 1: Time-Scale Separation and Convergence}
To prevent $\theta_k$ updates from destabilizing the physical system, the learning loop must enforce time-scale separation. The coupled system is modeled as a two-time-scale process: the fast physical plant dynamics ($x_k$) and the slow neural weight updates ($\theta_k$).

\textbf{Condition 1 (Fast-Loop $\delta$-ISS):} The retained baseline controller must enforce incremental Input-to-State Stability ($\delta$-ISS). This property ensures that the distance between any two system trajectories is strictly bounded by the magnitude of the difference in their inputs for any frozen parameter $\theta$. Consequently, the fast physical state $x_k$ converges exponentially to a parameter-dependent equilibrium manifold $x^*(\theta)$. This decouples the dynamics, permitting us to evaluate the loss function safely on the slow manifold, denoted as $\bar{L}(\theta)$.

\textbf{Condition 2 (Slow-Loop Lipschitz Continuity):} The reduced loss function $\bar{L}(\theta)$ must have a Lipschitz continuous gradient with a constant $L_\theta$, meaning that for any two parameter vectors $\theta_1$ and $\theta_2$, the change in the gradient is bounded by the distance between the parameters:
\begin{equation}
    \|\nabla_\theta \bar{L}(\theta_1) - \nabla_\theta \bar{L}(\theta_2)\| \leq L_\theta \|\theta_1 - \theta_2\|.
\end{equation}
In physical terms, this Lipschitz constant acts as a global ``speed limit'' on the curvature of the loss landscape, ensuring the gradient cannot change arbitrarily fast. Mathematically, this constant $L_\theta$ is bounded by the supremum of the spectral norm of the loss function's Hessian: $L_\theta = \mathop{\mathrm{sup}}_{\theta} \| \nabla_\theta^2 \bar{L}(\theta) \|_2$.

\begin{theorem}[Time-Scale Convergence]
If Conditions 1 and 2 hold, the parameter update law $\theta_{k+1} = \theta_k - \Gamma \nabla_\theta \bar{L}(\theta_k)$ will monotonically converge to a local optimum without inducing mechanical chatter, provided the positive definite learning rate matrix $\Gamma$ has a maximum eigenvalue that satisfies $\lambda_{\max}(\Gamma) < 2/L_\theta$.
\end{theorem}

\begin{proof}
By Condition 1, we evaluate the update on the equilibrium manifold $\bar{L}(\theta)$, bypassing the plant's transient dynamics. By Condition 2, the gradient $\nabla_\theta \bar{L}$ is $L_\theta$-Lipschitz continuous. Applying the standard Descent Lemma, which provides a guaranteed upper bound on the loss at the next iteration based on the maximum curvature, the new loss is bounded by the quadratic Taylor expansion:
{\small
\begin{equation*}
\begin{aligned}
    \bar{L}(\theta_{k+1}) \leq \bar{L}(\theta_k) &+ \nabla_\theta \bar{L}(\theta_k)^T (\theta_{k+1} - \theta_k) \\
    &+ 0.5 L_\theta \|\theta_{k+1} - \theta_k\|^2.
\end{aligned}
\end{equation*}
}
Substituting the parameter update law $(\theta_{k+1} - \theta_k) = -\Gamma \nabla_\theta \bar{L}(\theta_k)$ yields:
\begin{equation}
     \bar{L}(\theta_{k+1}) \leq \bar{L}(\theta_k) - \nabla_\theta \bar{L}(\theta_k)^T \left( \Gamma - \frac{L_\theta}{2} \Gamma^2 \right) \nabla_\theta \bar{L}(\theta_k).
\end{equation}
For monotonic descent (meaning the loss strictly decreases, $\bar{L}(\theta_{k+1}) < \bar{L}(\theta_k)$, for any non-zero gradient), the matrix $\left( \Gamma - \frac{L_\theta}{2} \Gamma^2 \right)$ must be positive definite. Factoring out $\Gamma$, we require $I - \frac{L_\theta}{2}\Gamma \succ 0$, which dictates that the maximum eigenvalue of the learning rate matrix must satisfy $\lambda_{\max}(\Gamma) < 2/L_\theta$.
\end{proof}

While Theorem 2 establishes a theoretical threshold for stability, computing the parameter Hessian's spectral norm ($L_\theta$) online is computationally intractable, and executing physical line-searches to find optimal step sizes is unsafe in robotic loops. To enforce this bound formally, deployments may adopt a \textit{dual-layer approach}, a two-pronged strategy combining offline structural constraints with online algorithmic safety filters.

In the first layer (offline), techniques like Spectral Normalization \cite{miyato2018spectral} can be used to bound the network's global Lipschitz constant, capping the network's potential volatility and ensuring a theoretical maximum for $L_\theta$ before the system is ever deployed. In the second layer (online), adaptive step-sizes are governed by curvature surrogates, such as the Generalized Gauss-Newton (GGN) or the empirical Fisher Information Matrix \cite{martens2020new}. These surrogates function as a dynamic braking mechanism, estimating the steepness of the learning landscape without the massive processing overhead of exact second-order calculations. Because the GGN approximates the Hessian using only first-order derivatives and remains positive semi-definite, fast power iterations can dynamically scale the learning rate matrix $\Gamma$. This dual-layer strategy ensures that the neural adaptation remains slow enough to satisfy time-scale separation, respecting the physical mechanics of the robot.

In practical nonlinear deployments, such as the saturated cubic plant evaluated in Section \ref{sec:illustrative_example}, analytically bounding the global Lipschitz constant $L_\theta$ and verifying $\delta$-ISS across the entire operational manifold is intractable. Consequently, the gradient clipping applied in our experiments acts as an empirical heuristic to enforce this slow-loop condition. As demonstrated in the subsequent ablation studies, removing this clipping bound violates the time-scale separation, causing the network to overreact to transient dynamics and inducing mechanical chatter.

Bounding $\lambda_{\max}(\Gamma)$ beneath this inverse curvature threshold ensures adaptation remains sufficiently slow. Violating this threshold breaches time-scale separation, coupling gradient updates with transient dynamics and destabilizing geometric safety bounds. This principle aligns with dual-loop architectures that pair fast inner safety constraints, e.g., Control Barrier Functions, with slower neural adaptation \cite{annaswamy2023integration,fisher2025analytical}.

As the online learner stabilizes the system, adapting on instantaneous gradients (utilizing a learning batch size of $\mathcal{B}_k = 1$) causes the network to act as a localized adaptive filter. While this yields high tracking performance on smooth, continuous limit cycles by overfitting to the local tangent dynamics, it induces catastrophic forgetting of the global state space \cite{Dohare2024}. For simple periodic tasks, this loss of global plasticity is benign. However, in discontinuous environments (e.g., hybrid switching in exoskeletons \cite{lippi2021icinco}), maintaining an experience replay buffer may become necessary to relax the requirement of persistent excitation and preserve geometric safety.

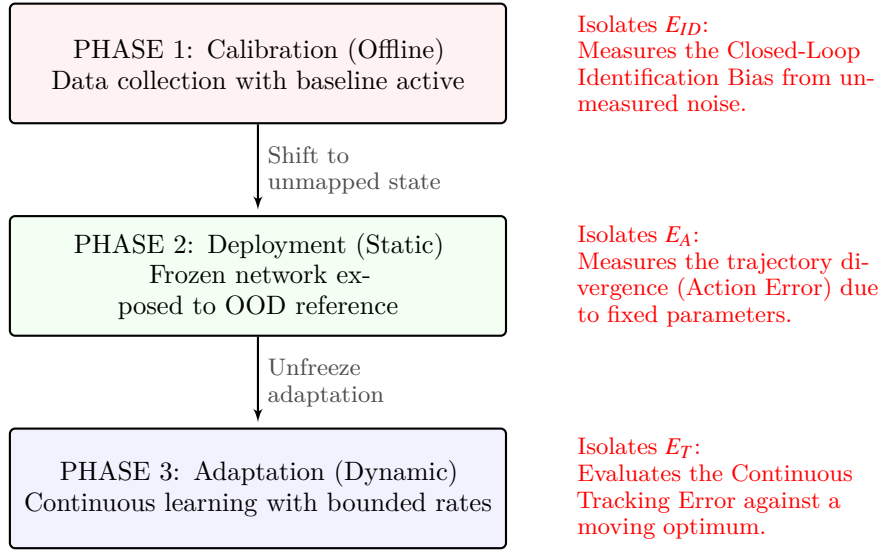
\begin{figure*}[!ht]
\centering
\begin{tikzpicture}[
    auto,
    >=latex',
    thick,
    phase/.style={draw, rectangle, minimum height=4.5em, text width=18em, align=center, fill=blue!5, rounded corners=2pt},
    offline/.style={phase, fill=red!5},
    static/.style={phase, fill=green!5},
    dynamic/.style={phase, fill=blue!5},
    error/.style={text=red, align=left, text width=12em, font=\small},
    line/.style={draw, ->, shorten >=2pt}
]

    \node[offline] (phase1) {
        \uppercase{Phase 1}: \textbf{Calibration (Offline)}\\
        \textit{Data collection with baseline active}
    };
    \node[error, right=0.8cm of phase1] (err_id) {
        \textbf{Isolates $E_{ID}$}:\\
        Measures the Closed-Loop Identification Bias from unmeasured noise.
    };

    \node[static, below=1.2cm of phase1] (phase2) {
        \uppercase{Phase 2}: \textbf{Deployment (Static)}\\
        \textit{Frozen network exposed to OOD reference}
    };
    \node[error, right=0.8cm of phase2] (err_ae) {
        \textbf{Isolates $E_A$}:\\
        Measures the trajectory divergence (Action Error) due to fixed parameters.
    };

    \node[dynamic, below=1.2cm of phase2] (phase3) {
        \uppercase{Phase 3}: \textbf{Adaptation (Dynamic)}\\
        \textit{Continuous learning with bounded rates}
    };
    \node[error, right=0.8cm of phase3] (err_te) {
        \textbf{Isolates $E_T$}:\\
        Evaluates the Continuous Tracking Error against a moving optimum.
    };

    \draw[line] (phase1) -- node[right, align=left, text=black!70, font=\small] {Shift to\\unmapped state} (phase2);
    \draw[line] (phase2) -- node[right, align=left, text=black!70, font=\small] {Unfreeze\\adaptation} (phase3);

\end{tikzpicture}
\caption{Experimental Evaluation Phases: The three-phase deployment strategy utilized in Section \ref{sec:illustrative_example} to explicitly isolate and quantify the proposed error taxonomy. While practical deployments may initialize directly into online adaptation, this staged experimental approach is necessary to independently evaluate the Closed-Loop Identification Error (Phase 1), the Action Error induced by static deployment (Phase 2), and the continuous Tracking Error mitigated by time-scale separated learning (Phase 3).}
\label{fig:experiment_phases}
\end{figure*}
\section{\uppercase{Illustrative Example}}
\label{sec:illustrative_example}

To validate the Action Error and the necessity of bounded continuous learning, we simulate a 1-DOF nonlinear Duffing oscillator. This system serves as a trackable model of wearable robotics (e.g., an exoskeleton joint exhibiting nonlinear ``hardening'' stiffness \cite{Schmalz2021.03.16.435713,paul2020dynamic}). It isolates the proposed taxonomy for geometric analysis without requiring novel controller designs.

The MATLAB source code for this simulation environment and all parallel ablation studies is publicly available at \url{https://github.com/VittorioFreiburg/Action-Error/}.

The discrete-time dynamics are governed by:
\begin{equation}
x_{k+1} = A x_k + B x_k^3 + u_k + w_k,
\end{equation}
where $A = 1.05$ represents the open-loop unstable linear dynamics, $B = 0.15$ introduces the unmodeled cubic stiffness, $u_k$ is the control input, and $w_k \sim \mathcal{N}(0, 0.025)$ represents stochastic environmental noise.

\textbf{Implementation Details.} The simulation and associated neural network architectures were implemented natively within the MATLAB environment, utilizing the Deep Learning Toolbox. The NN utilizes a single hidden layer (8 neurons, \text{tanh} activation), trained offline via Levenberg-Marquardt utilizing a full training batch size ($\mathcal{B} = \text{Full}$). Online adaptation operates via SGD evaluating a single sample per step ($\mathcal{B}_k = 1$) with gradient clipping bounded at $\|\nabla_\theta L\|_F \leq 0.5$. By restricting the maximum magnitude of the gradient step, this clipping acts as a safety filter to enforce the time-scale separation bound ($\lambda_{\max}(\Gamma) < 2/L_\theta$) established in Theorem 2, preventing the network from reacting aggressively to transient noise. In the context of the generalized control architecture presented in Figure \ref{fig:general_block_diagram}, the ``Plant'' block encapsulates this 1-DOF Duffing oscillator model. The physical state variable $x_k$, outputted by the plant, represents the angular displacement of the exoskeleton joint in radians, evaluated at discrete time steps $k$. To simulate mechanical joint limits and prevent numerical divergence, the plant's state update is subjected to a symmetric physical saturation limit of $\pm 5.0$ radians. While the generalized formulation in Section \ref{sec:formulation} treats the network as the unified controller, this hardware-inspired implementation retains the Proportional-Derivative (PD) baseline controller during deployment. The neural network acts as a parallel adaptive feedforward compensator, such that the total applied control effort is $u_k = u_k^{base}(x_k, x_{ref}) - g(x_k, x_{ref}, \theta)$.

The system undergoes a three-phase deployment (visually mapped in the experimental strategy of Figure \ref{fig:experiment_phases}):
\begin{enumerate}[label=\arabic*), leftmargin=*]
    \item \textbf{Calibration (Phase 1):} \looseness=-1 A baseline Proportional-Derivative (PD) controller ($K_p = 0.60, K_d = 0.25$) stabilizes the system to track a low-amplitude calibration sine wave ($x_{ref} = 0.3\sin(\pi k / 15)$) over a constrained dataset of $N = 150$ discrete samples. The plant is subjected to continuous Gaussian process noise $w_k \sim \mathcal{N}(0, 0.025^2)$. Because the state magnitude remains small, the cubic term $B x_k^3$ is negligible, leaving the nonlinearity hidden from offline identification.
\item \textbf{Deployment (Phase 2):} An offline neural network is trained on the Phase 1 residuals. The system is then tasked with a large-amplitude reference trajectory pushing the state into the unmapped nonlinear regime.
\item \textbf{Online Adaptation (Phase 3):} The network's weights are ``unfrozen'' and continuously updated online via standard Backpropagation. To re-establish Lyapunov stability, the gradients are bounded via clipping to enforce time-scale separation.
\end{enumerate}

To isolate the necessity of time-scale separation and gradient bounding, we perform parallel ablation studies on the Phase 3 online adaptation mechanism. To validate the proposed taxonomy, we calculate the three errors:

\begin{enumerate}[1),wide, labelwidth=!,labelindent=0pt]
    \item \textbf{Identification Error:} \looseness=-1 Evaluated offline, this represents the divergence between the network's prediction and the unmodeled dynamics over the calibration dataset $\mathcal{D}_{off}$. We quantify the estimation bias using an \textit{Estimated Bias Proxy} inspired by Instrumental Variable techniques, computing the ratio of the network output's Pearson correlation with the noisy baseline control to its correlation with the clean reference signal: $\text{Proxy} = |\rho(g_{\theta}, u^{base})| / (|\rho(g_{\theta}, x_{ref})| + \epsilon)$. In our evaluation, this yielded $0.041$. The variance explained by this noise correlation provides an expected lower-bound on the physical bias (the ``Ghost Variance''), calculated as $\text{Var}_{ghost} = \rho(g_{\theta}, u^{base})^2 \cdot \text{Var}(g_{\theta})$. This represents the footprint of the unmeasured noise that the neural network has internalized, confirming it mapped closed-loop artifacts rather than the deterministic plant dynamics.
   \item \textbf{Action Error (Generalization Gap).} Evaluated upon deployment, this is the difference between the network's training loss on $\mathcal{D}_{off}$ and its deployment loss on the out-of-distribution manifold $\mathcal{D}_{deploy}$. Bridging this metric with Theorem 1, we calculate the accumulation of the state sensitivity matrix ($S_k$). The Jacobians comprising $S_k$ are evaluated numerically at each time step via central finite differences ($\epsilon = 10^{-5}$) across the closed-loop plant and network parameters. While the physical trajectory in this representative execution settled into a sub-optimal offset, the underlying maximum sensitivity norm diverged ($\|S_k\|_F \approx 4.16 \times 10^{14}$). This divergence confirms unmitigated control authority. This numerical accumulation does not contradict the bounded LTI assumptions of Theorem 1, but illustrates the cascading global consequence of the Action Error. Theorem 1 guarantees the initial, local trajectory divergence. By shifting the state out-of-distribution, this bounded error acts as a catalyst, forcing the system into unmapped regimes where system-specific nonlinearities, such as the uncompensated cubic stiffness in this oscillator, take over. In these global regimes, local LTI approximations dissolve, allowing the tracking error to compound or diverge depending on the underlying physical topology of the plant.
    \item \textbf{Tracking Error:} Evaluated continuously during operation, this is the MSE between the physical closed-loop state $x_k$ and the reference limit cycle $x_{ref}$. Because this metric aggregates the entire deployment window, it captures the initial transient adaptation phase. Consequently, for the online architecture, this aggregate MSE serves as a conservative upper bound; over prolonged cyclic operation, the instantaneous tracking error asymptotically decays toward the system's stochastic noise floor.
\end{enumerate}

These metrics across the standard architectures and ablation models are summarized in Table \ref{tab:error_summary}.
Depending on stochastic noise realization, the unmitigated Action Error can cause the Offline NN to diverge (yielding undefined errors; e.g., maximum state sensitivity norm increasing to $\|S_k\| \approx 4.16 \times 10^{14}$). To allow meaningful comparison of bounded Tracking Errors, we executed a 100-iteration Monte Carlo simulation. Unmitigated Action Error caused the offline architecture to diverge in 64\% of trials. Table \ref{tab:error_summary} reports error metrics across the 36 stable executions (36\%).

As illustrated in Figure \ref{fig:ablation_results} for the Duffing oscillator and quantified in Table \ref{tab:error_summary}, the static offline network minimizes the initial Identification Error but fails to track the reference due to an Action Error. Conversely, the optimal online adaptation strategy bounds the trajectory shift and reduces the continuous Tracking Error. The ablation studies emphasize the necessity of Theorem 2: a sluggish learning rate fails to track the moving target, whereas an unstable learning rate breaches time-scale separation, inducing mechanical chatter. Divergent trials were excluded from these aggregations because the out-of-distribution trajectory divergence yields undefined or computationally infinite error metrics. This 64\% failure rate validates the Action Error: unmitigated static deployment is sub-optimal and hazardous to system stability.

\begin{table}[htpb]
\centering
\caption{Monte Carlo Summary of Experimental Error Metrics ($N=100$). Results report the mean and variance across the 36 stable runs.}
\label{tab:error_summary}
\resizebox{\columnwidth}{!}{%
\begin{tabular}{lccc}
\hline
\textbf{Architecture} & \textbf{Id. Error (MSE)} & \textbf{Action Error} & \textbf{Tracking Error} \\
\hline
Baseline PD            & - & - & 0.2037 ($\sigma^2$: 0.0010) \\
Static Offline NN      & 0.0001 ($\sigma^2$: 0.0000) & 0.0463 ($\sigma^2$: 0.0014) & 0.1218 ($\sigma^2$: 0.0103) \\
Optimal Online NN      & - & - & \textbf{0.0044} ($\sigma^2$: 0.0000) \\
Ablation: Sluggish LR  & - & - & 0.7096 ($\sigma^2$: 13.9977) \\
Ablation: Unstable LR  & - & - & 19.3110 ($\sigma^2$: 1.3640) \\
\hline
\end{tabular}%
}
\end{table}

\begin{figure*}[!ht]
  \centering
  \includegraphics[width=0.85\textwidth, trim = 2.4cm 0cm 2.0cm 0cm]{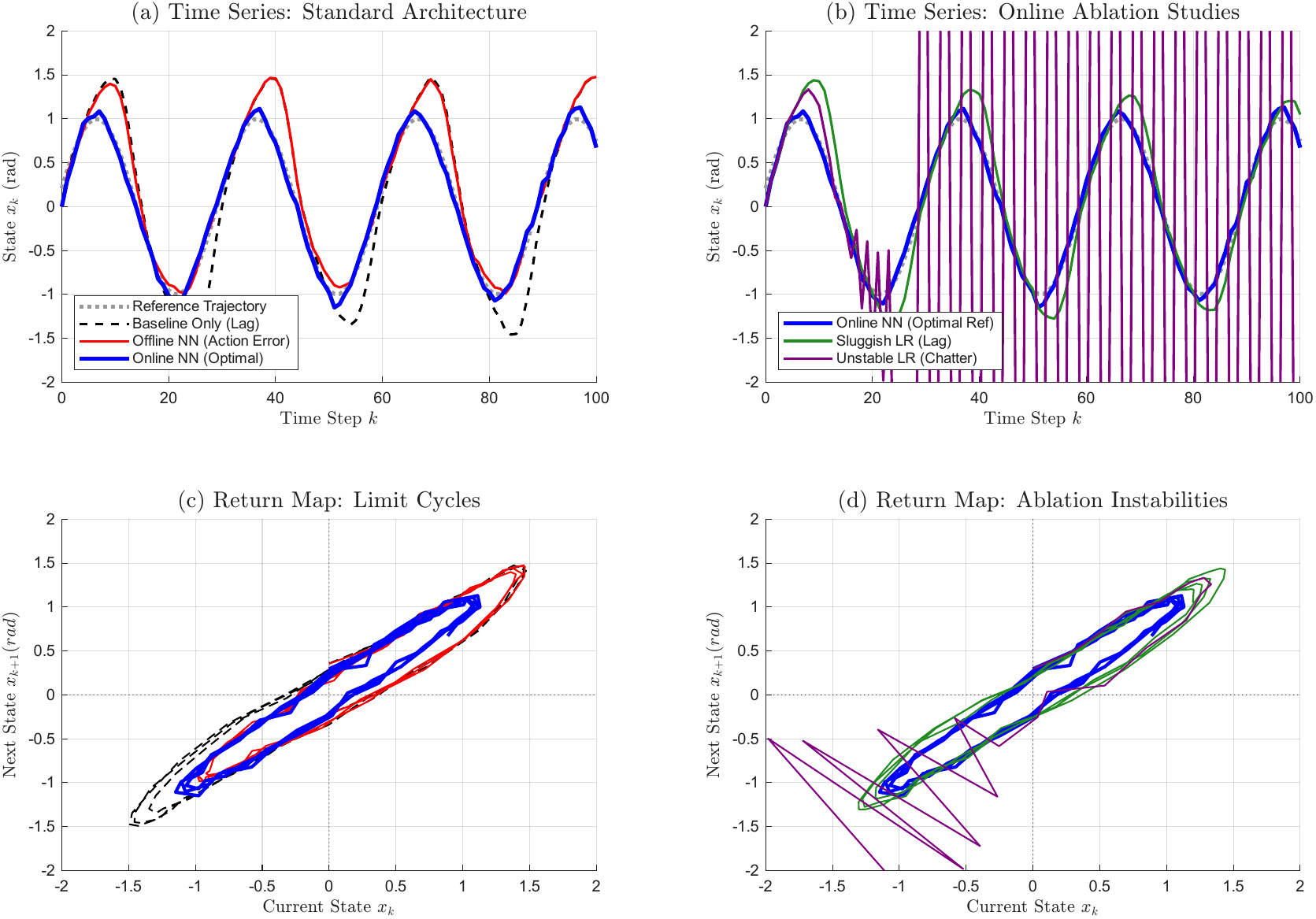}
  \caption{Simulation of the Action Error and Online Adaptation in a noisy nonlinear tracking task. \textbf{(a, c) Standard Architectures:} The Baseline PD (dashed black) exhibits phase lag due to uncompensated cubic stiffness. The Static Offline NN (red) pushes the system out-of-distribution, forcing an Action Error. The Optimal Online NN (blue) utilizes continuous gradient descent with a bounded learning rate ($\|\nabla_\theta L\| \leq 0.5$) to mitigate noise and track the limit cycle. \textbf{(b, d) Ablation Studies:} A sluggish learning rate ($\alpha = 0.01$, green) fails to adapt to the moving target, causing persistent phase lag. Conversely, applying the aggressive rate ($\alpha = 0.50$, purple) without gradient clipping violates time-scale separation; the network overreacts to transient errors, inducing mechanical chatter.}
  \label{fig:ablation_results}
\end{figure*}

\subsection{Comparative Benchmarking Against Classical and Structural Baselines}
\label{sec:comparative_benchmarking}

To contextualize the necessity of bounded continuous learning, we benchmark our proposed architecture against two established control paradigms: classical Model Reference Adaptive Control (MRAC) \cite{ioannou1996robust} and the structural Phase 1 safeguard of Dual Input-Output Parameterization (Dual IOP) \cite{chen2023dualiop}.

\textbf{Classical Baseline (MRAC):}
MRAC avoids the Action Error by utilizing a linear-in-the-parameters (LIP) formulation governed by Lyapunov update laws, ensuring global stability. However, when deployed on the Duffing oscillator, this structural rigidity becomes its primary limitation. The LIP formulation fails to capture the spatial complexity of the unmodeled cubic stiffness ($B x_k^3$). While the MRAC remains bounded and avoids trajectory divergence, it exhibits a persistent steady-state tracking offset with an amplitude distortion of 2.49\% and a Mean Squared Error (MSE) of 0.0062 (Table \ref{tab:reviewer_baselines}). This highlights the trade-off: classical methods guarantee geometric safety at the cost of representational capacity, whereas our time-scale separated NN recovers representational capacity while imposing safety via gradient bounding.

\textbf{Structural Baseline (Dual IOP):}
To isolate the Action Error from the initial closed-loop identification bias ($E_{ID}$), we implemented a Dual IOP framework during the Phase 1 calibration. By restricting the network's hypothesis space to models that commute with the stabilizing baseline controller, Dual IOP eliminated the ``Ghost Variance'' caused by the stochastic noise. However, upon freezing this identified model for Phase 2 static deployment, the system still exhibited trajectory divergence, yielding an amplitude distortion of 48.04\% and an MSE of 0.2122.

This empirical result (Figure \ref{fig:comparative_benchmarks}) validates the thesis of this paper: even if an offline model is identified without bias, the act of static deployment itself shifts the state distribution, rendering the optimal parameters a moving target. Structural identification safeguards like Dual IOP and ICI \cite{ici2025} are effective for calibration, but they do not eliminate the Action Error. Time-scale separated online adaptation remains a necessity for safe, continuous operation.

\begin{table}[htpb]
\centering
\caption{Empirical Tracking Error of Structural and Adaptive Baselines}
\label{tab:reviewer_baselines}
\resizebox{\columnwidth}{!}{%
\begin{tabular}{lcc}
\hline
\textbf{Architecture} & \textbf{Tracking Error (MSE)} & \textbf{Amplitude Distortion (\%)} \\
\hline
MRAC (Linear-in-the-Parameters) & 0.0062 & 2.49 \\
Dual IOP (Frozen Phase 2)       & 0.2122 & 48.04 \\
\hline
\end{tabular}%
}
\end{table}

\begin{figure*}[!ht]
\centering
\includegraphics[width=0.80\textwidth]{"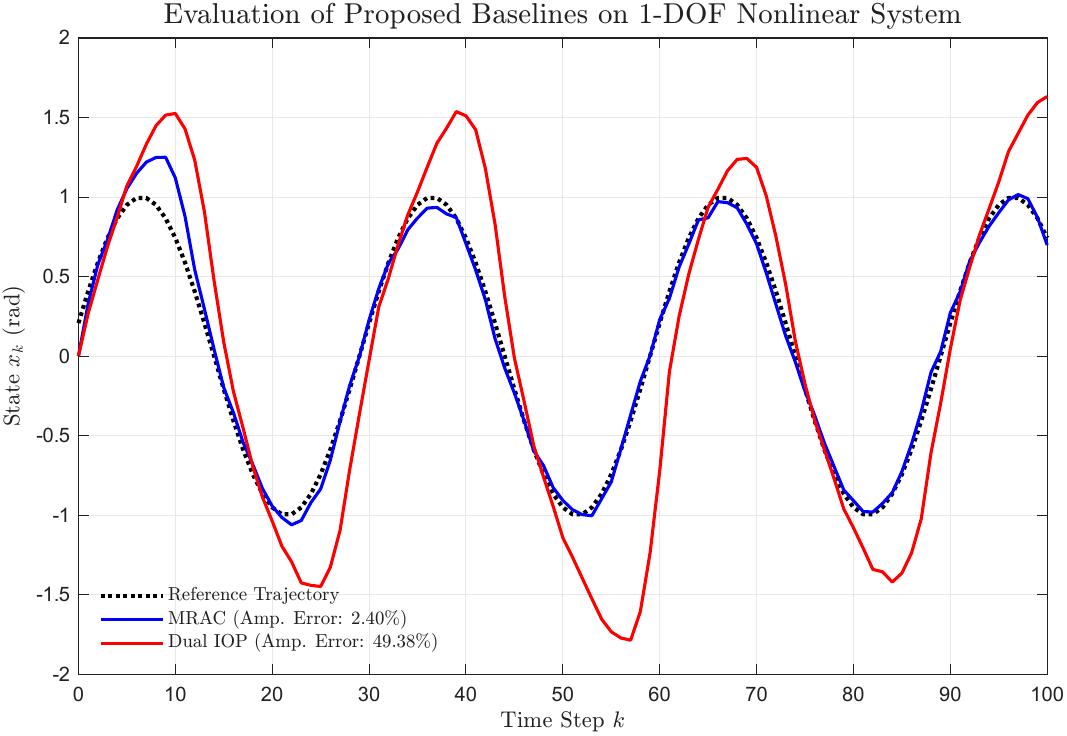"}
\caption{Evaluation of proposed baselines on the 1-DOF nonlinear Duffing oscillator. While the MRAC bounds the trajectory via its linear-in-the-parameters Lyapunov formulation, it suffers from a persistent tracking lag due to its inability to capture the unmodeled cubic stiffness. Conversely, the Dual IOP model, despite achieving unbiased closed-loop identification in Phase 1, distorts the amplitude upon static deployment, demonstrating the persistence of the Action Error independent of calibration bias.}
\label{fig:comparative_benchmarks}
\end{figure*}
\section{\uppercase{Discussion and Conclusions}}
\label{sec:discussion}
This paper formalized the Action Error (\S~\ref{sec:action_error}): the performance decay caused when an offline-trained NN alters the state distribution upon closed-loop deployment. This framework integrates the Closed-Loop Identification Error with the continuous Tracking Error required during online adaptation.

\textbf{The Ubiquity of Closed-Loop Dynamics.} \looseness=-1
The Action Error emerges from closed-loop control. While initially observed in unstable systems (e.g., legged robots), it persists when loops are closed via human-exoskeleton interaction \cite{lippi2021icinco}. Deployed machine learning models alter their operational environments. In generative AI, recursive training on synthetic data causes ``Model Collapse'' \cite{shumailov2024ai} or ``Model Autophagy Disorder'' (MAD) \cite{alemohammad2024self}, demonstrating how iterative feedback acts as a lossy compression mechanism that monotonically degrades data entropy \cite{stranak2026entropic}. Mitigating this, although delayed by accumulating human data \cite{gerstgrasser2024accumulating}, requires structural guardrails to break the closed-loop feedback cycle. This cycle of data degradation mirrors philosophical critiques regarding the epistemology of artificial intelligence. Literature introduces the concept of \emph{generative algorithmic epistemic injustice}, which investigates how generative AI can undermine the integrity of collective knowledge and the processes we rely on to acquire, assess, and trust information \cite{kay2024epistemic}. Grounded in social and political philosophy, this phenomenon identifies how AI systems can produce or amplify misinformation, perpetuate representational harm, and create epistemic inequities \cite{kay2024epistemic}. Whether in the physical state space of a robotic joint or the latent space of a large language model, the mechanism remains identical: without interventions to preserve diversity or enforce exploratory persistent excitation, the agent's control authority will degrade the environment it seeks to model.

\textbf{Breaking the Loop.}
Neurology offers inspiration for breaking these loops: sensory inputs are subject to natural detection thresholds of variable entity \cite{kobel2021vestibular}. In human postural control, non-linear thresholds and intermittent control paradigms for the inverted pendulum operate on integrated vestibular and proprioceptive signals \cite{hettich2013human,hettich2015human,tigrini2022long}, modalities mimicked in robotics by IMUs and joint encoders. Replicating these thresholds via deadbands can block sensory noise from propagating through the continuous control loop \cite{mergner2010neurological,10.3389/fnhum.2021.615200}, breaking the loop around defined safe zones to suppress control authority without endangering stability \cite{icinco20}. Beyond rigid thresholds, cognitive neuroscience provides a framework for loop-breaking via Predictive Processing and Active Inference \cite{friston2017active,clark2013whatever}. In these frameworks, embodied agents minimize variational free energy (``surprise'') by altering the physical world to conform to internal hierarchical predictions \cite{lanillos2021active}. To avoid becoming trapped in feedback loops, biological systems dynamically modulate the ``precision'' (reliability) of sensory prediction errors \cite{pezzulo2015active}. In robotic control, engineering this adaptive precision-weighting acts as a neuromorphic deadband, allowing the controller to suppress transient sensorimotor errors \cite{lanillos2021active,oliver2022empirical}. This theoretical relationship is bidirectional: while active inference uses precision-weighting as a functional threshold, physical deadband models, such as the human-derived Disturbance Estimation and Compensation (DEC) framework, produce sensory reweighting as an emergent macroscopic behavior without requiring explicit gain scheduling \cite{lippi2017human}. This biologically plausible symmetry empowers the agent to balance stabilization with epistemic exploration, ensuring it does not overreact to its action-induced sensory noise \cite{lanillos2021active,oliver2022empirical,friston2017active}.

\textbf{Embodiment and the Limits of Offline Learning.}
The mathematical inevitability of the Action Error exposes a fundamental limit of offline learning: it treats the neural network as a disembodied artifact operating within an epistemic vacuum. Stabilizing a complex dynamic system mandates structural and physical embodiment, requiring continuous, situated interaction with the physical environment \cite{mergner2024embodiment}. Statically deploying a frozen model commits a fallacy of functional anthropomorphization \cite{mergner2024embodiment}; it expects robust adaptability while denying the system the continuous physical experience, or historical embodiment, necessary to co-evolve with a shifting state space \cite{mergner2024embodiment,mainzer2009embodied}. However, continuous adaptation risks recursive feedback degradation if the agent indiscriminately tracks its own self-produced errors. Safe historical co-evolution requires the agent to distinguish external physical disturbances from the trajectory shifts induced by its own parametric sub-optimality, a control-theoretic requirement that mirrors the biological somatosensory sense of self-agency (SSA) \cite{mergner2024embodiment}. By enforcing time-scale separation, the architecture achieves this balance: it maintains the continuous structural coupling required for historical embodiment \cite{mergner2024embodiment,mainzer2009embodied}, while mathematically bounding the network's authority to prevent it from recursively chasing its own sensory noise.

\textbf{Classical Adaptive Control and Robust MPC.} \looseness=-1
Architectures like Model Reference Adaptive Control (MRAC) \cite{ioannou1996robust,naranjo2025adaptive} and $\mathcal{L}_1$ Adaptive Control \cite{hovakimyan2010l1} mitigate parameter drift using robust bounds and low-pass filtering. They rely on linear-in-the-parameters (LIP) structures or continuous online tuning, which contrasts with static, non-convex offline models. While Robust Model Predictive Control (MPC) \cite{kouvaritakis2016robust} circumvents offline bias through online optimization, supervised neural approximations of these MPC laws \cite{fontes2009new} face spatial curvature artifacts that necessitate our proposed Hessian regularization and experience replay guardrails.

\textbf{Reinforcement Learning.}
These findings contextualize Reinforcement Learning (RL) \cite{sutton2018reinforcement}. While online RL circumvents identification bias through active interaction, offline RL remains vulnerable to the Action Error via out-of-distribution extrapolation \cite{fujimoto2019off,levine2020offline}. Continuous RL tracks moving targets but relies on exploratory noise, which is dangerous for safety-critical hardware like exoskeletons \cite{garcia2015comprehensive}, reintroducing the tension between Lyapunov stability and persistent excitation \cite{berkenkamp2017safe}.

\textbf{Structural Parameterizations, Certificates, and Advanced Safety Wrappers.}
Advancements in closed-loop system identification, such as Dual IOP \cite{chen2023dualiop}, Neural Identification \cite{boroujeni2025neural}, and the nonlinear generalization of Independent Component Identification (ICI) \cite{ici2025}, mitigate the initial Identification Bias by guaranteeing stabilized model extraction from closed-loop data. Mitigating the subsequent Action Error requires run-time safety. Certificate learning frameworks provide learned Lyapunov or Control Barrier Functions (CBFs) \cite{dawson2022safe}, which, like inherent Lyapunov structures \cite{min2023datadrivencontrolinherentlyapunov}, act as fast-loop spatial safety filters, while Safe Bayesian Optimization (BO) \cite{safeBO2024} provides probabilistic constraints during closed-loop tuning. Because these structural methods prioritize guaranteed bounds over spatial representation recovery, our proposed time-scale separated adaptation ($\lambda_{\max}(\Gamma) < 2/L_\theta$) does not replace them; rather, it acts as a slow-loop companion. By bounding the temporal update rate, we ensure neural adaptation remains sufficiently sluggish so as not to trigger or destabilize the fast-loop spatial constraints provided by CBFs or Safe BO. For deployment, this adaptation can be integrated with Seq-AMPC wrappers for provable constraint satisfaction \cite{reviewer_seqampc} or analyzed within NARX surrogate MPC stability bounds \cite{reviewer_narx}. Physics-informed IDA-PBC formulations \cite{reviewer_idapbc} and reservoir computing's ``mode-budgeting'' \cite{reviewer_reservoir} can absorb residual tracking errors and suppress unstable authority outside trained regions. Finally, integrating Instrumental Variable (IV)-based data-driven predictive control (DDPC) \cite{ivddpc2023} alongside our empirical bias proxy offers a structured pathway to validate closed-loop consistency prior to deployment, while bridging continuous SGD with formal non-PE adaptive identification schemes disarms the Identifiability Trap \cite{reviewer_nonPE}.

\textbf{Time-Varying Sensitivities, Noise, and Hybrid Systems.}
Although Theorem 1 establishes the Action Error via a local LTI approximation, deployments involve dynamic Jacobians that demand incremental stability ($\delta$-ISS) or bounded-variation arguments. In non-autonomous regimes, unmeasured process noise amplifies parametric sub-optimality, and discontinuous hybrid switching thresholds in exoskeletons \cite{lippi2021icinco} can induce impulse-like jumps in state sensitivity.

\textbf{Limitations and Future Work.} While the 1-DOF Duffing oscillator provides a geometric testbed to isolate the Action Error without the confounding variables of multi-DOF dynamic coupling, the present work functions as an analytical framework. Scaling this time-scale separated adaptation to multi-DOF hardware (e.g., lower-limb exoskeletons or humanoid robots) introduces coupled, time-varying Jacobians. In these higher-dimensional regimes, the sensitivity matrix $S_{k}$ becomes non-diagonal, and unmodeled cross-joint inertia amplifies parametric sub-optimality. The gradient clipping threshold must be scaled by the spectral norm of the coupled mass matrix to ensure the slow-loop Lipschitz continuity (Condition 2) is maintained across all joint axes simultaneously. To elevate this taxonomy to a comprehensive control methodology, future work prioritizes extending empirical validation to such hardware-in-the-loop deployments. This requires experimental benchmarking of our time-scale separated adaptation against classical robust architectures (MRAC, $\mathcal{L}_1$) and sequential aggregation (DAgger). Theoretically, while Remark 1 addresses continuous saturation-induced instability, efforts are required to formalize sensitivity bounds under discontinuous hybrid switching and time-varying references.

%

\bibliographystyle{apalike-ejor}
\setlength{\bibhang}{10pt}
\let\oldbibliography\thebibliography
\renewcommand{\thebibliography}[1]{%
  \oldbibliography{#1}%
  \setlength{\itemsep}{0pt}
  \setlength{\parskip}{0pt}
}
{\small
\bibliography{exampleShort}}

@Article{10759774,
  author   = {Lhoste, Clément and Barış Küçüktabak, Emek and Vianello, Lorenzo and Amato, Lorenzo and Short, Matthew R. and Lynch, Kevin M. and Pons, Jose L.},
  title    = {Deep-Learning Estimation of Weight Distribution Using Joint Kinematics for Lower-Limb Exoskeleton Control},
  journal  = {IEEE Transactions on Medical Robotics and Bionics},
  year     = {2025},
  volume   = {7},
  number   = {1},
  pages    = {20-26},
  doi      = {10.1109/TMRB.2024.3503922},
}

@Article{10556541,
  author   = {Belal, Mohammad and Alsheikh, Nael and Aljarah, Ahmad and Hussain, Irfan},
  title    = {Deep Learning Approaches for Enhanced Lower-Limb Exoskeleton Control: A Review},
  journal  = {IEEE Access},
  year     = {2024},
  volume   = {12},
  pages    = {143883-143907},
  doi      = {10.1109/ACCESS.2024.3414175},
}

@Article{tigrini2022long,
  author    = {Tigrini, Andrea and Verdini, Federica and Fioretti, Sandro and Mengarelli, Alessandro},
  title     = {Long term correlation and inhomogeneity of the inverted pendulum sway time-series under the intermittent control paradigm},
  journal   = {Communications in Nonlinear Science and Numerical Simulation},
  year      = {2022},
  volume    = {108},
  pages     = {106198},
  publisher = {Elsevier},
}

@Book{sutton2018reinforcement,
  title     = {Reinforcement learning: An introduction},
  publisher = {MIT press},
  year      = {2018},
  author    = {Sutton, Richard S and Barto, Andrew G},
}

@InProceedings{fujimoto2019off,
  author    = {Fujimoto, Scott and Meger, David and Precup, Doina},
  title     = {Off-Policy Deep Reinforcement Learning without Exploration},
  booktitle = {Proceedings of the 36th International Conference on Machine Learning},
  year      = {2019},
  editor    = {Chaudhuri, Kamalika and Salakhutdinov, Ruslan},
  volume    = {97},
  series    = {Proceedings of Machine Learning Research},
  pages     = {2052--2062},
  month     = {09--15 Jun},
  publisher = {PMLR},
  url       = {https://proceedings.mlr.press/v97/fujimoto19a.html},
}

@Misc{levine2020offline,
  author        = {Sergey Levine and Aviral Kumar and George Tucker and Justin Fu},
  title         = {Offline Reinforcement Learning: Tutorial, Review, and Perspectives on Open Problems},
  year          = {2020},
  archiveprefix = {arXiv},
  eprint        = {2005.01643},
  primaryclass  = {cs.LG},
  url           = {https://arxiv.org/abs/2005.01643},
}

@Article{garcia2015comprehensive,
  author  = {Garc{\'\i}a, Javier and Fern{\'a}ndez, Fernando},
  title   = {A comprehensive survey on safe reinforcement learning},
  journal = {Journal of Machine Learning Research},
  year    = {2015},
  volume  = {16},
  number  = {1},
  pages   = {1437--1480},
}

@InProceedings{berkenkamp2017safe,
  author    = {Berkenkamp, Felix and Turchetta, Matteo and Schoellig, Angela and Krause, Andreas},
  title     = {Safe model-based reinforcement learning with stability guarantees},
  booktitle = {Neural Information Processing Systems},
  year      = {2017},
  volume    = {30},
}

@Book{ljung1999system,
  title     = {System Identification: Theory for the User},
  publisher = {Prentice Hall},
  year      = {1999},
  author    = {Lennart Ljung},
  edition   = {2nd},
}

@Article{soderstrom2007errors,
  author  = {Torsten S{\"o}derstr{\"o}m},
  title   = {Errors-in-variables methods in system identification},
  journal = {Automatica},
  year    = {2007},
  volume  = {43},
  number  = {6},
  pages   = {929--958},
}

@Article{gilson2005instrumental,
  author  = {Marion Gilson and Paul M. J. Van den Hof},
  title   = {Instrumental variable methods for closed-loop system identification},
  journal = {Automatica},
  year    = {2005},
  volume  = {41},
  number  = {2},
  pages   = {241--249},
}

@Article{wensing2024optimization,
  author  = {Patrick M. Wensing and Michael Posa and Yue Hu and Adrien Escande and Nicolas Mansard and Andrea Del Prete},
  title   = {Optimization-based control for dynamic legged robots},
  journal = {IEEE Transactions on Robotics},
  year    = {2024},
  volume  = {40},
  pages   = {44--67},
}

@InProceedings{zhang2024humanoid,
  author       = {Zhang, Qiang and Cui, Peter and Yan, David and Sun, Jingkai and Duan, Yiqun and Han, Gang and Zhao, Wen and Zhang, Weining and Guo, Yijie and Zhang, Arthur and others},
  title        = {Whole-body humanoid robot locomotion with human reference},
  booktitle    = {2024 IEEE/RSJ International Conference on Intelligent Robots and Systems (IROS)},
  year         = {2024},
  pages        = {11225--11231},
  organization = {IEEE},
}

@Misc{li2024okami,
  author        = {Jinhan Li and Yifeng Zhu and Yuqi Xie and Zhenyu Jiang and Mingyo Seo and Georgios Pavlakos and Yuke Zhu},
  title         = {OKAMI: Teaching Humanoid Robots Manipulation Skills through Single Video Imitation},
  year          = {2024},
  archiveprefix = {arXiv},
  eprint        = {2410.11792},
  primaryclass  = {cs.RO},
  url           = {https://arxiv.org/abs/2410.11792},
}

@Article{lippi2018prediction,
  author    = {Vittorio Lippi},
  title     = {Prediction in the Context of a Human-Inspired Posture Control Model},
  journal   = {Robotics and Autonomous Systems},
  year      = {2018},
  volume    = {107},
  pages     = {63--70},
  month     = {10},
  issn      = {0921-8890},
  doi       = {10.1016/j.robot.2018.05.012},
  publisher = {Elsevier},
  url       = {https://www.sciencedirect.com/science/article/pii/S0921889017309107},
}

@InProceedings{ross2011reduction,
  author    = {St{\'e}phane Ross and Geoffrey Gordon and Drew Bagnell},
  title     = {A reduction of imitation learning and structured prediction to no-regret online learning},
  booktitle = {Proceedings of the Fourteenth International Conference on Artificial Intelligence and Statistics (AISTATS)},
  year      = {2011},
  pages     = {627--635},
}

@Article{naranjo2025adaptive,
  author  = {D. M. Naranjo and A. M. Annaswamy},
  title   = {Adaptive Control of Dubins Vehicle in the Presence of Loss of Effectiveness},
  journal = {IEEE Control Systems Letters},
  year    = {2025},
}

@Misc{fisher2025analytical,
  author = {P. A. Fisher and A. M. Annaswamy},
  title  = {Analytical Construction of CBF-Based Safety Filters for Simultaneous State and Input Constraints},
  year   = {2025},
  note   = {arXiv:2409.19352},
  eprint = {2409.19352},
}

@Article{forssell1999closedloop,
  author  = {Urban Forssell and Lennart Ljung},
  title   = {Closed-loop identification revisited},
  journal = {Automatica},
  year    = {1999},
  volume  = {35},
  number  = {7},
  pages   = {1215--1241},
  doi     = {10.1016/S0005-1098(99)00022-9},
}

@Article{pillonetto2025deep,
  author  = {Gianluigi Pillonetto and Aleksandr Y. Aravkin and Daniel Gedon and Lennart Ljung and Ant{\^o}nio H. Ribeiro and Thomas B. Sch{\"o}n},
  title   = {Deep networks for system identification: A survey},
  journal = {Automatica},
  year    = {2025},
  volume  = {171},
  pages   = {111907},
  doi     = {10.1016/j.automatica.2024.111907},
}

@InProceedings{lippi2021icinco,
  author       = {Vittorio Lippi and Cristian Camardella and Alessandro Filippeschi and Francesco Porcini},
  title        = {{Identification of Gait Phases with Neural Networks for Smooth Transparent Control of a Lower Limb Exoskeleton}},
  booktitle    = {Proceedings of the 18th International Conference on Informatics in Control, Automation and Robotics (ICINCO)},
  year         = {2021},
  pages        = {171--178},
  organization = {INSTICC},
  publisher    = {SciTePress},
  doi          = {10.5220/0010554401710178},
  isbn         = {978-989-758-522-7},
}

@Article{annaswamy2023integration,
  author  = {Anuradha M. Annaswamy and Anubhav Guha and Yingnan Cui and Sunbochen Tang and Peter A. Fisher and Joseph E. Gaudio},
  title   = {Integration of Adaptive Control and Reinforcement Learning for Real-Time Control and Learning},
  journal = {IEEE Transactions on Automatic Control},
  year    = {2023},
  volume  = {68},
  number  = {12},
  pages   = {7740--7755},
  doi     = {10.1109/TAC.2023.3290037},
}

@Misc{annaswamy2023lessons,
  author       = {Anuradha M. Annaswamy},
  title        = {Lessons from Adaptive Control: Towards Real-Time Machine Learning},
  howpublished = {Keynote Address at the 20th International Conference on Informatics in Control, Automation and Robotics (ICINCO)},
  year         = {2023},
  note         = {Available at: \url{https://icinco.scitevents.org/KeynoteSpeakers.aspx?y=2023}},
  address      = {Rome, Italy},
}

@Article{narendra1990identification,
  author    = {K. S. Narendra and K. Parthasarathy},
  title     = {Identification and Control of Dynamical Systems Using Neural Networks},
  journal   = {IEEE Transactions on Neural Networks},
  year      = {1990},
  volume    = {1},
  number    = {1},
  pages     = {4--27},
  doi       = {10.1109/72.80202},
  publisher = {IEEE},
}

@Article{shumailov2024ai,
  author  = {Ilia Shumailov and Zakhar Shumaylov and Yiren Zhao and Yarin Gal and Nicolas Papernot and Ross Anderson},
  title   = {AI models collapse when trained on recursively generated data},
  journal = {Nature},
  year    = {2024},
  volume  = {631},
  pages   = {755--759},
  doi     = {10.1038/s41586-024-07566-y},
}

@Misc{lovland2025feedback,
  author        = {Kristian Lindb{\"a}ck L{\o}vland and Lars Imsland and Bjarne Andr{\'e} Grimstad},
  title         = {On a Closed-Loop Identification Challenge in Feedback Optimization},
  year          = {2025},
  archiveprefix = {arXiv},
  eprint        = {2509.01188},
  primaryclass  = {math.OC},
}

@misc{chen2023dualiop,
  author = {Ran Chen and Amber Srivastava and Mingzhou Yin and Roy S. Smith},
  title = {Closed-Loop Identification of Stabilized Models Using Dual Input-Output Parameterization},
  year = {2023},
  eprint = {2311.09019},
  archivePrefix = {arXiv},
  primaryClass = {math.OC}
}

@misc{srivastava2023dualslp,
  author = {Amber Srivastava and Mingzhou Yin and Andrea Iannelli and Roy S. Smith},
  title = {A Dual System-Level Parameterization for Identification from Closed-Loop Data},
  year = {2023},
  eprint = {2304.02379},
  archivePrefix = {arXiv},
  primaryClass = {math.OC}
}

@article{gonzalez2024consistency,
  author = {Rodrigo A. Gonz{\'a}lez and Siqi Pan and Cristian R. Rojas},
  title = {Consistency analysis of refined instrumental variable methods for continuous-time system identification in closed-loop},
  journal = {Automatica},
  volume = {166},
  year = {2024}
}

@Misc{recursive2025iv,
  author        = {Koen Classens and Rodrigo A. González and Tom Oomen},
  title         = {Recursive Identification of Structured Systems: An Instrumental-Variable Approach Applied to Mechanical Systems},
  year          = {2025},
  archiveprefix = {arXiv},
  eprint        = {2504.17927},
  primaryclass  = {eess.SY},
  url           = {https://arxiv.org/abs/2504.17927},
}

@Misc{min2023datadrivencontrolinherentlyapunov,
  author        = {Youngjae Min and Spencer M. Richards and Navid Azizan},
  title         = {Data-Driven Control with Inherent Lyapunov Stability},
  year          = {2023},
  archiveprefix = {arXiv},
  eprint        = {2303.03157},
  primaryclass  = {eess.SY},
  url           = {https://arxiv.org/abs/2303.03157},
}

@Misc{boroujeni2025neural,
  author        = {Mahrokh G. Boroujeni and Laura Meroi and Leonardo Massai and Clara L. Galimberti and Giancarlo Ferrari-Trecate},
  title         = {Neural Identification of Feedback-Stabilized Nonlinear Systems},
  year          = {2025},
  archiveprefix = {arXiv},
  eprint        = {2503.22601},
  primaryclass  = {eess.SY},
}

@InProceedings{10644154,
  author    = {Basyal, Sujata and Ting, Jonathan and Mishra, Kislaya and Allen, Brendon C.},
  title     = {Concurrent Learning and Lyapunov-Based Updates of Deep Neural Networks for Euler-Lagrange Dynamic Systems},
  booktitle = {2024 American Control Conference (ACC)},
  year      = {2024},
  pages     = {15--20},
  doi       = {10.23919/ACC60939.2024.10644154},
}

@Article{Dohare2024,
  author  = {Dohare, Shibhansh and Hernandez-Garcia, J. Fernando and Lan, Qingfeng and Rahman, Parash and Mahmood, A. Rupam and Sutton, Richard S.},
  title   = {Loss of plasticity in deep continual learning},
  journal = {Nature},
  year    = {2024},
  volume  = {632},
  number  = {8026},
  pages   = {768--774},
  doi     = {10.1038/s41586-024-07711-7},
}

@InProceedings{mamou2014neural,
  author       = {M. A. Mamou and N. Saadia},
  title        = {Neural Networks Controler of a Lower Limbs Robotic Rehabilitation Chair},
  booktitle    = {Proceedings of the International Conference on Biomedical Electronics and Devices - Volume 1: BIODEVICES, (BIOSTEC 2014)},
  year         = {2014},
  pages        = {65-71},
  organization = {INSTICC},
  publisher    = {SciTePress},
  doi          = {10.5220/0004700700650071},
  isbn         = {978-989-758-013-0},
}

@InProceedings{fontes2009new,
  author       = {A. {de Barros Fontes} and P. Amorim and M. {Ribeiro da Silva Garcia}},
  title        = {A New Approach of the Neural Predictive Control Applied in a Thermoelectric Module},
  booktitle    = {Proceedings of the 6th International Conference on Informatics in Control, Automation and Robotics - Volume 3: ICINCO,},
  year         = {2009},
  pages        = {189-194},
  organization = {INSTICC},
  publisher    = {SciTePress},
  doi          = {10.5220/0002193001890194},
  isbn         = {978-989-8111-99-9},
}

@Article{gustavsson1977identification,
  author  = {Ivar Gustavsson and Lennart Ljung and Torsten S{\"o}derstr{\"o}m},
  title   = {Identification of processes in closed loop--identifiability and accuracy aspects},
  journal = {Automatica},
  year    = {1977},
  volume  = {13},
  number  = {1},
  pages   = {59--75},
  note    = {Originally presented at the 4th IFAC Symposium on Identification and System Parameter Estimation, Tbilisi, 1976},
  doi     = {10.1016/0005-1098(77)90008-8},
}

@Article{martens2020new,
  author  = {Martens, James},
  title   = {New Insights and Perspectives on the Natural Gradient Method},
  journal = {Journal of Machine Learning Research},
  year    = {2020},
  volume  = {21},
  number  = {146},
  pages   = {1--76},
  url     = {http://jmlr.org/papers/v21/17-678.html},
}

@InProceedings{miyato2018spectral,
  author    = {Miyato, Takeru and Kataoka, Toshiki and Koyama, Masanori and Yoshida, Yuichi},
  title     = {Spectral Normalization for Generative Adversarial Networks},
  booktitle = {International Conference on Learning Representations (ICLR)},
  year      = {2018},
}

@Book{ioannou1996robust,
  title     = {Robust Adaptive Control},
  publisher = {Dover Publications, Incorporated},
  year      = {2012},
  author    = {Ioannou, P.A. and Sun, J.},
  series    = {Dover Books on Electrical Engineering Series},
  isbn      = {9780486498171},
  lccn      = {2012022707},
  url       = {https://books.google.de/books?id=pXWFY_vbg1MC},
}

@Book{hovakimyan2010l1,
  title     = {$\mathcal{L}_1$ Adaptive Control Theory: Guaranteed Robustness with Fast Adaptation},
  publisher = {Society for Industrial and Applied Mathematics},
  year      = {2010},
  author    = {Hovakimyan, Naira and Cao, Chengyu},
  series    = {Advances in Design and Control},
  isbn      = {9780898717044},
  lccn      = {2010013646},
  url       = {https://books.google.de/books?id=U2Giej-1bukC},
}

@Book{kouvaritakis2016robust,
  title     = {Model Predictive Control: Classical, Robust and Stochastic},
  publisher = {Springer},
  year      = {2016},
  author    = {Kouvaritakis, Basil and Cannon, Mark},
  series    = {Advanced Textbooks in Control and Signal Processing},
  address   = {Cham},
  isbn      = {978-3-319-24851-6},
  doi       = {10.1007/978-3-319-24853-0},
}

@Misc{reviewer_narx,
  author        = {Lea Bold and Irene Schimperna and Karl Worthmann and Johannes Köhler},
  title         = {Exponential stability of data-driven nonlinear MPC based on input/output models},
  year          = {2026},
  archiveprefix = {arXiv},
  eprint        = {2603.16808},
  primaryclass  = {math.OC},
  url           = {https://arxiv.org/abs/2603.16808},
}

@PhdThesis{reviewer_idapbc,
  author = {Alessio Moreschini},
  title  = {Modeling and control of discrete-time and sampled-data port-Hamiltonian systems},
  school = {Sapienza University of Rome},
  year   = {2021},
  note   = {See Section 5.2.1: Constructive aspects: an approximate IDA-PBC design},
}

@Misc{reviewer_nonPE,
  author        = {Rebecca G. Hart and Omkar Sudhir Patil and Zachary I. Bell and Warren E. Dixon},
  title         = {System Identification and Control Using Lyapunov-Based Deep Neural Networks without Persistent Excitation: A Concurrent Learning Approach},
  year          = {2025},
  archiveprefix = {arXiv},
  eprint        = {2505.10678},
  primaryclass  = {eess.SY},
  url           = {https://arxiv.org/abs/2505.10678},
}

@Misc{reviewer_reservoir,
  author        = {Satoshi Oishi and Hiroshi Yamashita and Hideyuki Suzuki and Sho Shirasaka},
  title         = {Stabilizing chaotic dynamical system reproduction in reservoir computing},
  year          = {2026},
  archiveprefix = {arXiv},
  eprint        = {2602.11069},
  primaryclass  = {nlin.CD},
  url           = {https://arxiv.org/abs/2602.11069},
}

@Misc{reviewer_seqampc,
  author        = {Mihaela-Larisa Clement and M{\'o}nika Farsang and Agnes Poks and Johannes Edelmann and Manfred Pl{\"o}chl and Radu Grosu and Ezio Bartocci},
  title         = {Towards Safe Learning-Based Non-Linear Model Predictive Control through Recurrent Neural Network Modeling},
  year          = {2026},
  archiveprefix = {arXiv},
  eprint        = {2603.24503},
  primaryclass  = {cs.LG},
  url           = {https://arxiv.org/abs/2603.24503},
}

@InProceedings{icinco20,
  author    = {Vittorio Lippi and Fabio Molinari},
  title     = {{Lyapunov Stability of a Nonlinear Bio-inspired System for the Control of Humanoid Balance}},
  booktitle = {Proc. 17th Int. Conf. on Informatics in Control, Autom. and Robotics (ICINCO)},
  year      = {2020},
  pages     = {726--733},
  doi       = {10.5220/0009970307260733},
}

@MastersThesis{paul2020dynamic,
  author = {Paul, Elizabeth Napoleon},
  title  = {{Dynamic Analysis of Compliant Joint for Assistive Exoskeleton}},
  school = {Aalborg University},
  year   = {2020},
}

@Misc{Schmalz2021.03.16.435713,
  author       = {Schmalz, Jelena and Paul, David and Shorter, Kathleen and Schmalz, Xenia and Cooper, Matthew and Murphy, Aron},
  title        = {{Modelling Human Gait using a Nonlinear Differential Equation}},
  year         = {2021},
  doi          = {10.1101/2021.03.16.435713},
  elocation-id = {2021.03.16.435713},
  eprint       = {https://www.biorxiv.org/content/early/2021/07/25/2021.03.16.435713.full.pdf},
  journal      = {bioRxiv},
  publisher    = {Cold Spring Harbor Laboratory},
  url          = {https://www.biorxiv.org/content/early/2021/07/25/2021.03.16.435713},
}

@InProceedings{alemohammad2024self,
  author    = {Alemohammad, Sina and others},
  title     = {Self-Consuming Generative Models Go MAD},
  booktitle = {Proc. Int. Conf. on Learning Representations (ICLR)},
  year      = {2024},
}

@Misc{stranak2026entropic,
  author    = {Pavel Stra{\v{n}}{\'a}k},
  title     = {{Entropic Limits of Iterative Computation in Generative AI: Model Collapse Explained by the Data Processing Inequality and the AI Theorem}},
  month     = {April},
  year      = {2026},
  doi       = {10.20944/preprints202507.2260.v2},
  journal   = {Preprints},
  publisher = {Preprints.org},
  url       = {https://doi.org/10.20944/preprints202507.2260.v2},
}

@InProceedings{gerstgrasser2024accumulating,
  author    = {Gerstgrasser, Matthias and others},
  title     = {{Accumulating Data Avoids Model Collapse}},
  booktitle = {NeurIPS 2024 Workshop on Mathematics of Modern Machine Learning},
  year      = {2024},
  url       = {https://openreview.net/forum?id=4DYFHAcgw3},
}

@Article{kobel2021vestibular,
  author    = {Kobel, Megan J and Wagner, Andrew R and Merfeld, Daniel M and Mattingly, Jameson K},
  title     = {{Vestibular thresholds: a review of advances and challenges in clinical applications}},
  journal   = {Frontiers in Neurology},
  year      = {2021},
  volume    = {12},
  pages     = {643634},
  publisher = {Frontiers Media SA},
}

@InProceedings{hettich2013human,
  author       = {Hettich, Georg and Lippi, Vittorio and Mergner, Thomas},
  title        = {Human-like sensor fusion mechanisms in a postural control robot},
  booktitle    = {Special Session on Sensory Fusion for Diagnostics and Neurorehabilitation},
  year         = {2013},
  volume       = {2},
  pages        = {152--160},
  organization = {SCITEPRESS},
}

@Article{mergner2010neurological,
  author   = {T. Mergner},
  title    = {A neurological view on reactive human stance control},
  journal  = {Annual Reviews in Control},
  year     = {2010},
  volume   = {34},
  number   = {2},
  pages    = {177-198},
  issn     = {1367-5788},
  doi      = {https://doi.org/10.1016/j.arcontrol.2010.08.001},
  url      = {https://www.sciencedirect.com/science/article/pii/S1367578810000404},
}

@Article{10.3389/fnhum.2021.615200,
  author   = {Akçay, Mustafa Emre and others},
  title    = {Visual Modulation of Human Responses to Support Surface Translation},
  journal  = {Front Hum Neurosci},
  year     = {2021},
  volume   = {Volume 15 - 2021},
  issn     = {1662-5161},
  doi      = {10.3389/fnhum.2021.615200},
  url      = {https://www.frontiersin.org/journals/human-neuroscience/articles/10.3389/fnhum.2021.615200},
}

@InProceedings{ici2025,
  author    = {Boroujeni, Mahrokh G. and Meroi, Laura and Massai, Leonardo and Galimberti, Clara L. and Ferrari-Trecate, Giancarlo},
  title     = {{Neural Identification of Feedback-Stabilized Nonlinear Systems}},
  booktitle = {2025 IEEE 64th Conference on Decision and Control (CDC)},
  year      = {2025},
  pages     = {4201--4207},
  doi       = {10.1109/CDC57313.2025.11312335},
}

@Article{dawson2022safe,
  author  = {Dawson, Charles and Gao, Sicun and Fan, Chuchu},
  title   = {Safe Control With Learned Certificates: A Survey of Neural Lyapunov, Barrier, and Contraction Methods for Robotics and Control},
  journal = {IEEE Transactions on Robotics},
  year    = {2023},
  volume  = {39},
  number  = {3},
  pages   = {1749--1767},
  doi     = {10.1109/TRO.2022.3232542},
}

@InProceedings{safeBO2024,
  author    = {Hirt, Sebastian and Pfefferkorn, Maik and Findeisen, Rolf},
  title     = {Safe and Stable Closed-Loop Learning for Neural-Network-Supported Model Predictive Control},
  booktitle = {2024 IEEE 63rd Conference on Decision and Control (CDC)},
  year      = {2024},
  doi       = {10.1109/CDC56724.2024.10886624},
}

@Article{ivddpc2023,
  author  = {Wang, Yibo and Qiu, Yiwen and Sader, Malika and Huang, Dexian and Shang, Chao},
  title   = {Data-Driven Predictive Control Using Closed-Loop Data: An Instrumental Variable Approach},
  journal = {IEEE Control Systems Letters},
  year    = {2023},
  volume  = {7},
  pages   = {3639--3644},
  doi     = {10.1109/LCSYS.2023.3340444},
}

@InCollection{hettich2015human,
  author    = {Hettich, Georg and Lippi, Vittorio and Mergner, Thomas},
  title     = {{Human-Like Sensor Fusion Implemented in the Posture Control of a Bipedal Robot}},
  booktitle = {Neurotechnology, Electronics, and Informatics},
  publisher = {Springer},
  year      = {2015},
  pages     = {29--45},
}

@InProceedings{kay2024epistemic,
  author    = {Kay, Jackie and Kasirzadeh, Atoosa and Mohamed, Shakir},
  title     = {Epistemic Injustice in Generative AI},
  booktitle = {Proceedings of the 2024 AAAI/ACM Conference on AI, Ethics, and Society},
  year      = {2025},
  series    = {AIES '24},
  pages     = {684–697},
  publisher = {AAAI Press},
  location  = {San Jose, California, USA},
  numpages  = {14},
}

@Article{friston2017active,
  author   = {Friston, Karl and FitzGerald, Thomas and Rigoli, Francesco and Schwartenbeck, Philipp and Pezzulo, Giovanni},
  title    = {{Active Inference: A Process Theory}},
  journal  = {Neural Computation},
  year     = {2017},
  volume   = {29},
  number   = {1},
  pages    = {1-49},
  month    = {01},
  issn     = {0899-7667},
  doi      = {10.1162/NECO_a_00912},
  eprint   = {https://direct.mit.edu/neco/article-pdf/29/1/1/984132/neco_a_00912.pdf},
  url      = {https://doi.org/10.1162/NECO_a_00912},
}

@Article{clark2013whatever,
  author  = {Clark, Andy},
  title   = {Whatever next? Predictive brains, situated agents, and the future of cognitive science},
  journal = {Behavioral and Brain Sciences},
  year    = {2013},
  volume  = {36},
  number  = {3},
  pages   = {181–204},
  doi     = {10.1017/S0140525X12000477},
}

@Article{lippi2017human,
  author    = {Lippi, Vittorio and Mergner, Thomas},
  title     = {Human-derived disturbance estimation and compensation ({DEC}) method lends itself to a modular sensorimotor control in a humanoid robot},
  journal   = {Frontiers in Neurorobotics},
  year      = {2017},
  volume    = {11},
  pages     = {49},
  doi       = {10.3389/fnbot.2017.00049},
  publisher = {Frontiers Media SA},
}

@Misc{lanillos2021active,
  author        = {Pablo Lanillos and Cristian Meo and Corrado Pezzato and Ajith Anil Meera and Mohamed Baioumy and Wataru Ohata and Alexander Tschantz and Beren Millidge and Martijn Wisse and Christopher L. Buckley and Jun Tani},
  title         = {Active Inference in Robotics and Artificial Agents: Survey and Challenges},
  year          = {2021},
  archiveprefix = {arXiv},
  eprint        = {2112.01871},
  primaryclass  = {cs.RO},
  url           = {https://arxiv.org/abs/2112.01871},
}

@Article{pezzulo2015active,
  author    = {Pezzulo, Giovanni and Rigoli, Francesco and Friston, Karl},
  title     = {Active inference, epistemic value, and vicarious trial and error},
  journal   = {Learning \& Memory},
  year      = {2015},
  volume    = {22},
  number    = {6},
  pages     = {322--338},
  publisher = {Cold Spring Harbor Lab},
}

@Article{oliver2022empirical,
  author   = {Oliver, Guillermo and Lanillos, Pablo and Cheng, Gordon},
  title    = {An Empirical Study of Active Inference on a Humanoid Robot},
  journal  = {IEEE Transactions on Cognitive and Developmental Systems},
  year     = {2022},
  volume   = {14},
  number   = {2},
  pages    = {462-471},
  doi      = {10.1109/TCDS.2021.3049907},
}

@InBook{mergner2024embodiment,
  pages     = {553--579},
  title     = {Embodiment and Humanoid Robotics},
  publisher = {Springer Fachmedien Wiesbaden},
  year      = {2024},
  author    = {Mergner, Thomas and Funk, Michael and Lippi, Vittorio},
  editor    = {Mainzer, Klaus},
  address   = {Wiesbaden},
  isbn      = {978-3-658-19606-6},
  booktitle = {Philosophisches Handbuch K{\"u}nstliche Intelligenz},
  doi       = {10.1007/978-3-658-19606-6_23},
  url       = {https://doi.org/10.1007/978-3-658-19606-6_23},
}

@Article{mainzer2009embodied,
  author   = {Klaus Mainzer},
  title    = {From embodied mind to embodied robotics: Humanities and system theoretical aspects},
  journal  = {Journal of Physiology-Paris},
  year     = {2009},
  volume   = {103},
  number   = {3},
  pages    = {296-304},
  issn     = {0928-4257},
  note     = {Neurorobotics},
  doi      = {https://doi.org/10.1016/j.jphysparis.2009.08.012},
  url      = {https://www.sciencedirect.com/science/article/pii/S0928425709000539},
}

\end{document}